\documentclass[11pt,a4paper]{article}
\usepackage[hyperref]{emnlp2019}
\usepackage{times}
\usepackage{latexsym}
\usepackage{url}

\aclfinalcopy 
\usepackage{helvet}
\usepackage{courier}
\usepackage{multirow}
\usepackage{bm}
\usepackage{amsmath}
\usepackage{amsthm,amssymb}
\usepackage{mdwlist}
\usepackage{algorithm}
\usepackage[noend]{algpseudocode}
\usepackage{relsize}
\usepackage{amsfonts}
\usepackage{graphicx}
\usepackage{subfigure}
\usepackage{caption}
\usepackage{hhline}
\usepackage{color}
\usepackage{colortbl}
\usepackage{booktabs}
\newtheorem{theorem}{Theorem}

\makeatletter
\def\BState{\State\hskip-\ALG@thistlm}
\makeatother

\DeclareMathOperator*{\argmax}{arg\,max}

\title{Weighed Domain-Invariant Representation Learning for Cross-domain Sentiment Analysis}

\author{
Minlong Peng, Qi Zhang, Xuanjing Huang\\
School of Computer Science, Fudan University\\
825 Zhangheng Road, Shanghai, China\\
\{mlpeng16,qz,xjhuang\}@fudan.edu.cn
}
 \date{}
 
\begin{document}
\maketitle
\begin{abstract}



Cross-domain sentiment analysis is currently a hot topic in the research and engineering areas. One of the most popular frameworks in this field is the domain-invariant representation learning (DIRL) paradigm, which aims to learn a distribution-invariant feature representation across domains. However, in this work, we find out that applying DIRL may harm domain adaptation when the label distribution $\rm{P}(\rm{Y})$ changes across domains. To address this problem, we propose a modification to DIRL, obtaining a novel weighted domain-invariant representation learning (WDIRL) framework. We show that it is easy to transfer existing SOTA DIRL models to WDIRL. Empirical studies on extensive cross-domain sentiment analysis tasks verified our statements and showed the effectiveness of our proposed solution.

\end{abstract} 

\section{Introduction}

Sentiment analysis aims to predict sentiment polarity of user-generated data with emotional orientation like movie reviews. The exponentially increase of online reviews makes it an interesting topic in research and industrial areas. However, reviews can span so many different domains and the collection and preprocessing of large amounts of data for new domains is often time-consuming and expensive. Therefore, cross-domain sentiment analysis is currently a hot topic, which aims to transfer knowledge from a label-rich source domain (S) to the label-few target domain (T). 

In recent years, one of the most popular frameworks for cross-domain sentiment analysis is the domain invariant representation learning (DIRL) framework \cite{glorot2011domain,fernando2013unsupervised,ganin2016domain,zellinger2017central,
li2017end}. Methods of this framework follow the idea of extracting a domain-invariant feature representation, in which the data distributions of the source and target domains are similar. Based on the resultant representations, they learn the supervised classifier using source rich labeled data. 
The main difference among these methods is the applied technique to force the feature representations to be domain-invariant.

However, in this work, we discover that applying DIRL may harm domain adaptation in the situation that the label distribution $\rm{P}(\rm{Y})$ shifts across domains. Specifically, let $\rm{X}$ and $\rm{Y}$ denote the input and label random variable, respectively, and $G(\rm{X})$ denote the feature representation of $\rm{X}$. We found out that when $\rm{P}(\rm{Y})$ changes across domains while $\rm{P}(\rm{X}|\rm{Y})$ stays the same, forcing $G(\rm{X})$ to be domain-invariant will make $G(\rm{X})$ uninformative to $\rm{Y}$. This will, in turn, harm the generation of the supervised classifier to the target domain. In addition, for the more general condition that both $\rm{P}(\rm{Y})$ and $\rm{P}(\rm{X}|\rm{Y})$ shift across domains, we deduced a conflict between the object of making the classification error small and that of making $G(\rm{X})$ domain-invariant. 

We argue that the problem is worthy of studying since the shift of $\rm{P}(\rm{Y})$ exists in many real-world cross-domain sentiment analysis tasks \cite{glorot2011domain}. For example, the marginal distribution of the sentiment of a product can be affected by the overall social environment and change in different time periods; and for different products, their marginal distributions of the sentiment are naturally considered different. Moreover, there are many factors, such as the original data distribution, data collection time, and data clearing method, that can affect $\rm{P}(\rm{Y})$ of the collected target domain unlabeled dataset. Note that in the real-world cross-domain tasks, we do not know the labels of the collected target domain data. Thus, we cannot previously align its label distribution $\rm{P}_T(\bm{Y})$ with that of source domain labeled data $\rm{P}_S(\bm{Y})$, as done in many previous works \cite{glorot2011domain,ganin2016domain,tzeng2017adversarial,li2017end,he2018adaptive,peng2018cross}. 



To address the problem of DIRL resulted from the shift of $\rm{P}(\rm{Y})$, we propose a modification to DIRL, obtaining a weighted domain-invariant representation learning (WDIRL) framework. This framework additionally introduces a class weight $\bm{w}$ to weigh source domain examples by class, hoping to make $\rm{P}(\rm{Y})$ of the weighted source domain close to that of the target domain. Based on $\bm{w}$, it resolves domain shift in two steps. In the first step, it forces the marginal distribution $\rm{P}(\rm{X})$ to be domain-invariant between the target domain and the weighted source domain instead of the original source, obtaining a supervised classifier $\rm{P}_S(\rm{Y}|\rm{X}; \bm{\Phi})$ and a class weight $\bm{w}$. In the second step, it resolves the shift of $\rm{P}(\rm{Y}|\rm{X})$ by adjusting $\rm{P}_S(\rm{Y}|\rm{X}; \bm{\Phi})$ using $\bm{w}$ for label prediction in the target domain. We detail these two steps in \textsection \ref{sec:wdirl}. Moreover, we will illustrate how to transfer existing DIRL models to their WDIRL counterparts, taking the representative metric-based CMD model \cite{zellinger2017central} and the adversarial-learning-based DANN model \cite{ganin2016domain} as an example, respectively.

In summary, the contributions of this paper include: ($\bm{i}$) We theoretically and empirically analyse the problem of DIRL for domain adaptation when the marginal distribution $\rm{P}(\rm{Y})$ shifts across domains. ($\bm{ii}$) We proposed a novel method to address the problem and show how to incorporate it with existent DIRL  models. ($\bm{iii}$) Experimental studies on extensive cross-domain sentiment analysis tasks show that models of our WDIRL framework can greatly outperform their DIRL counterparts.

\section{Preliminary and Related Work}
\subsection{Domain Adaptation}
For expression consistency, in this work, we consider domain adaptation in the unsupervised setting (however, we argue that our analysis and solution also applies to the supervised and semi-supervised domain adaptation settings). In the unsupervised domain adaptation setting, there are two different distributions over $\rm{X} \times \rm{Y}$: the source domain $\rm{P}_S(\rm{X},\rm{Y})$ and the target domain $\rm{P}_T(\rm{X},\rm{Y})$. And there is a labeled data set $\mathcal{D}_S$ drawn $i.i.d$ from $\rm{P}_S(\rm{X},\rm{Y})$ and an unlabeled data set $\mathcal{D}_T$ drawn $i.i.d.$ from the marginal distribution $\rm{P}_T(\rm{X})$:
\begin{align*}
    & \mathcal{D}_S=\{(x_i, y_i)\}_{i=1}^n \sim \rm{P}_S(\rm{X}, \rm{Y}),\\
    & \mathcal{D}_T=\{x_i\}_{i=n+1}^N \sim \rm{P}_T(\rm{X}).
\end{align*}
The goal of domain adaptation is to build a classier $f:\rm{X} \rightarrow \rm{Y}$ that has good performance in the target domain using $\mathcal{D}_S$ and $\mathcal{D}_T$. 

For this purpose, many approaches have been proposed from different views, such as instance reweighting~\cite{mansour2009domain}, pivot-based information passing \cite{blitzer2007biographies}, spectral feature alignment \cite{pan2010cross}
subsampling~\cite{chen2011automatic}, and of course the domain-invariant representation learning~\cite{pan2011domain,gopalan2011domain,long2013transfer,muandet2013domain,yosinski2014transferable,long2015domain,aljundi2015landmarks,wei2016deep,bousmalis2016domain,pinheiro2018unsupervised,zhao2018adversarial}. 

\subsection{Domain Invariant Representation Learning}

Domain invariant representation learning (DIRL) is a very popular framework for performing domain adaptation in the cross-domain sentiment analysis field \cite{ghifary2014domain,li2017end,chen2018adversarial,peng2018cross}. It is heavily motivated by the following theorem \cite{ben2007analysis}.
\begin{theorem} \label{theorem:basic}
For a hypothesis $h$,
\begin{equation}\label{eq:previous_theorem}
\begin{split}
&\mathcal{L}_T(h) \leq \mathcal{L}_S(h) + {d}_{1}(\rm{P}_S(\rm{X}), \rm{P}_T(\rm{X})) \\
&+ \min\{\mathbb{E}_{\bm{x} \sim \rm{P}_S}\left[|\rm{P}_S(y|\bm{x})-\rm{P}_T(y|\bm{x})|\right], \\ 
&\quad \quad \quad \; \; \mathbb{E}_{\bm{x} \sim \rm{P}_T}\left[|\rm{P}_S(y|\bm{x})-\rm{P}_T(y|\bm{x})|\right] \},
\end{split}
\end{equation}
Here, $\mathcal{L}_S(h)$ denotes the expected loss with hypothesis $h$ in the source domain, $\mathcal{L}_T(h)$ denotes the counterpart in the target domain, $d_1$ is a measure of divergence between two distributions.
\end{theorem}

Based on Theorem \ref{theorem:basic} and assuming that performing feature transform on $\rm{X}$ will not increase the values of the first and third terms of the right side of Ineq. (\ref{eq:previous_theorem}), methods of the DIRL framework apply a feature map $G$ onto $\rm{X}$, hoping to obtain a feature representation $G(\rm{X})$ that has a lower value of ${d}_{1}(\rm{P}_S(G(\rm{X})), \rm{P}_T(G(\rm{X})))$. To this end, different methods have been proposed. These methods can be roughly divided into two directions. The first direction is to design a differentiable metric to explicitly evaluate the discrepancy between two distributions. We call methods of this direction as the metric-based DIRL methods. A representative work of this direction is the center-momentum-based model proposed by \citet{zellinger2017central}. In that work, they proposed a central moment discrepancy metric (CMD) to evaluate the discrepancy between two distributions. Specifically, let denote $\rm{X}_S$ and $\rm{X}_T$ an $M$ dimensional random vector on the compact interval $[a; b]^M$ over distribution $\rm{P}_S$ and $\rm{P}_T$, respectively. The CMD loss between $\rm{P}_S$ and $\rm{P}_T$ is defined by:
\begin{equation}
\begin{split}
\text{CMD}_K&(\rm{X}_S, \rm{X}_T) = \frac{1}{|b-a|} \parallel E(\rm{X}_S) - E(\rm{X}_T) \parallel_2 \\
&+ \frac{1}{|b-a|^k} \sum_{k=2}^{K} \parallel C_k(\rm{X}_S) - C_k(\rm{X}_T) \parallel_2.
\end{split}
\end{equation}
Here, $\mathbb{E}(\rm{X})$ denotes the expectation of $\rm{X}$ over distribution $\rm{P}_S(\rm{X})$, and
$$C_k(X)=\left( \mathbb{E}(\prod_{i=1}^M (\rm{X}_i-\mathbb{E}(\rm{X}_i))^{r_i}\right)_{r_i \geq 0, \sum_i^M r_i=k},$$
is the $k$-th momentum, where $\rm{X}_i$ denotes the $i^{th}$ dimensional variable of $\rm{X}$.

The second direction is to perform adversarial training between the feature generator $G$ and a domain discriminator $D$. We call methods of this direction as the adversarial-learning-based methods. As a representative, \citet{ganin2016domain} trained $D$ to distinguish the domain of a given example $x$ based on its representation $G(x)$. At the same time, they encouraged $G$ to deceive $D$, i.e., to make $D$ unable to distinguish the domain of $x$. More specifically, $D$ was trained to minimize the loss:
\begin{equation}
\begin{split}
\mathcal{L}_d &= \mathbb{E}_{x \sim \rm{P}_S(\rm{X})}\left[\log \frac{1}{D(G(x))}\right] \\
&+ \mathbb{E}_{x \sim \rm{P}_T(\rm{X})}\left[\log\frac{1}{1- D(G(x))}\right] 
\end{split}
\end{equation}
over its trainable parameters, while in contrast $G$ was trained to maximize $\mathcal{L}_d$. 
According to the work of \citet{Goodfellow2014Generative}, this is equivalent to minimize the Jensen-shannon divergence \cite{amari1987differential,lin1991divergence} $\text{JSD}(\rm{P}_S, \rm{P}_T)$ between $\rm{P}_S(G(\rm{X}))$ and $\rm{P}_T(G(\rm{X}))$ over $G$. Here, for a concise expression, we write $\rm{P}$ as the shorthand for $\rm{P}(G(\rm{X}))$.

The task loss is the combination of the supervised learning loss $\mathcal{L}_{sup}$ and the domain-invariant learning loss $\mathcal{L}_{inv}$, which are defined on $\mathcal{D}_S$ only and on the combination of $\mathcal{D}_S$ and $\mathcal{D}_T$, respectively:
\begin{equation}\label{eq:general_loss}
\mathcal{L} = \mathcal{L}_{sup}(\mathcal{D}_S) + \alpha \mathcal{L}_{inv}(\mathcal{D}_S, \mathcal{D}_T).
\end{equation}
Here, $\alpha$ is a hyper-parameter for loss balance, and the aforementioned domain adversarial loss $\text{JSD}(\rm{P}_S, \rm{P}_T)$ and $\text{CMD}_K$ are two concrete forms of $\mathcal{L}_{inv}$.

\section{Problem of Domain-Invariant Representation Learning}

In this work, we found out that applying DIRL may harm domain adaptation in the situation that $\rm{P}(\rm{Y})$ shifts across domains. Specifically, when $\rm{P}_S(\rm{Y})$ differs from $\rm{P}_T(\rm{Y})$, forcing the feature representations $G(\rm{X})$ to be domain-invariant may increase the value of $\mathcal{L}_S(h)$ in Ineq. (\ref{eq:previous_theorem}) and consequently increase the value of $\mathcal{L}_T(h)$, which means the decrease of target domain performance. In the following, we start our analysis under the condition that $\rm{P}_S(\rm{X}|\rm{Y})=\rm{P}_T(\rm{X}|\rm{Y})$. Then, we consider the more general condition that $\rm{P}_S(\rm{X}|\rm{Y})$ also differs from $\rm{P}_T(\rm{X}|\rm{Y})$.

When $\rm{P}_S(\rm{X}|\rm{Y})=\rm{P}_T(\rm{X}|\rm{Y})$, we have 
the following theorem.
\begin{theorem} \label{theorem:target_shift}
Given $\rm{P}_S(\rm{X}|\rm{Y})=\rm{P}_T(\rm{X}|\rm{Y})$, if $\rm{P}_S(\rm{Y}=i) \neq \rm{P}_T(\rm{Y}=i)$ and a feature map $G$ makes $\rm{P}_S \left( \mathcal{M}(\rm{X}))=\rm{P}_T(\mathcal{M}(\rm{X}) \right)$, then $\rm{P}_S(\rm{Y}=i|\mathcal{M}(\rm{X}))=\rm{P}_S(\rm{Y}=i)$.
\end{theorem}

\begin{proof}
Proofs appear in Appendix \textbf{A}.
\end{proof}

\paragraph{Remark.}
According to Theorem \ref{theorem:target_shift}, we know that when $\rm{P}_S(\rm{X}|\rm{Y})=\rm{P}_T(\rm{X}|\rm{Y})$ and $\rm{P}_S(\rm{Y}=i) \neq \rm{P}_T(\rm{Y}=i)$, forcing $G(\rm{X})$ to be domain-invariant inclines to make data of class $i$ mix with data of other classes in the space of $G(\rm{X})$. This will make it difficult for the supervised classifier to distinguish inputs of class $i$ from inputs of the other classes. \textit{Think about such an extreme case that every instance $x$ is mapped to a consistent point $g_0$ in $G(\rm{X})$}. In this case, $\rm{P}_S(G(\rm{X})=g_0)= \rm{P}_T(G(\rm{X})=g_0) = 1$. Therefore, $G(\rm{X})$ is domain-invariant. As a result, the supervised classifier will assign the label $y^* = \argmax_y \rm{P}_S(\rm{Y}=y)$ to all input examples. This is definitely unacceptable. 
To give a more intuitive illustration of the above analysis, we offer several empirical studies on Theorem \ref{theorem:target_shift} in Appendix \textbf{B}.

When $\rm{P}_S(\rm{Y})\neq \rm{P}_T(\rm{Y})$ and $\rm{P}_S(\rm{X}|\rm{Y}) \neq \rm{P}_T(\rm{X}|\rm{Y})$, we did not obtain such a strong conclusion as Theorem \ref{theorem:target_shift}. Instead, we deduced a conflict between the object of achieving superior classification performance and that of making features domain-invariant.

Suppose that $\rm{P}_S(\rm{Y}=i) \neq \rm{P}_T(\rm{Y}=i)$ and instances of class $i$ are completely distinguishable from instances of the rest classes in $G(\rm{X})$, i.e.,:
\begin{align*}
    &\resizebox{\columnwidth}{!}{$\rm{P}(G(\rm{X}=x)|\rm{Y}=i) > 0 \Rightarrow \rm{P}(G(\rm{X}=x)|\rm{Y}\neq i) = 0$} \\
    &\resizebox{\columnwidth}{!}{$\rm{P}(G(\rm{X}=x)|\rm{Y}\neq i) > 0 \Rightarrow \rm{P}(G(\rm{X}=x)|\rm{Y} = i) = 0$}.
\end{align*}
In DIRL, we hope that: $$\resizebox{\columnwidth}{!}{$\sum_{i=1}^{L}\rm{P}_S(G(\rm{X})|\rm{Y}=i)\rm{P}_S(\rm{Y}=i)=\sum_{i=1}^{L}\rm{P}_T(G(\rm{X})|\rm{Y}=i)\rm{P}_T(\rm{Y}=i)$}.$$
Consider the region $x \in \mathcal{X}_i$, where $\rm{P}(G(\rm{X}=x)|\rm{Y}=i)>0$. According to the above assumption, we know that $\rm{P}(G(\rm{X}=x \in \mathcal{X}_i)|\rm{Y} \neq i) = 0$. Therefore, applying DIRL will force 
\begin{equation*}
    \resizebox{\columnwidth}{!}{$\rm{P}_S(G(\rm{X}=x)|\rm{Y}=i)\rm{P}_S(\rm{Y}=i) = \rm{P}_T(G(\rm{X}=x)|\rm{Y}=i)\rm{P}_T(\rm{Y}=i)$}
\end{equation*}
in region $x \in \mathcal{X}_i$. Taking the integral of $x$ over $\mathcal{X}_i$ for both sides of the equation, we have $\rm{P}_S(\rm{Y}=i) = \rm{P}_T(\rm{Y}=i)$. This deduction contradicts with the setting that $\rm{P}_S(\rm{Y}=i) \neq \rm{P}_T(\rm{Y}=i)$. Therefore, $G(\rm{X})$ is impossible fully class-separable when it is domain-invariant. Note that the object of the supervised learning is exactly to make $G(\rm{X})$ class-separable. Thus, this actually indicates a conflict between the supervised learning and the domain-invariant representation learning. 

Based on the above analysis, we can conclude that \textit{it is impossible to obtain a feature representation $G(X)$ that is class-separable and at the same time, domain-invariant using the DIRL framework, when $\rm{P}(\rm{Y})$ shifts across domains.} However, the shift of $\rm{P}(\rm{Y})$ can exist in many cross-domain sentiment analysis tasks. Therefore, it is worthy of studying in order to deal with the problem of DIRL.

\section{Weighted Domain Invariant Representation Learning} \label{sec:wdirl}
According to the above analysis, we proposed a weighted version of DIRL to address the problem caused by the shift of $\rm{P}(\rm{Y})$ to DIRL. The key idea of this framework is
to first align $\rm{P}(\rm{Y})$ across domains before performing domain-invariant learning, and then take account the shift of $\rm{P}(\rm{Y})$ in the label prediction procedure. Specifically, it introduces a class weight $\bm{w}$ to weigh source domain examples by class. Based on the weighted source domain, the domain shift problem is resolved in two steps. In the first step, it applies DIRL on the target domain and the weighted source domain, aiming to alleviate the influence of the shift of $\rm{P}(\rm{Y})$ during the alignment of $\rm{P}(\rm{X}|\rm{Y})$. In the second step, it uses $\bm{w}$ to reweigh the supervised classifier $\rm{P}_S(\rm{Y}|\rm{X})$ obtained in the first step for target domain label prediction. We detail these two steps in \textsection \ref{sec:p_x_y} and \textsection \ref{sec:p_y}, respectively. 

\subsection{Align $\rm{P}(\rm{X}|\rm{Y})$ with Class Weight} \label{sec:p_x_y}
The motivation behind this practice is to adjust data distribution of the source domain or the target domain to alleviate the shift of $\rm{P}(\rm{Y})$ across domains before applying DIRL. Consider that we only have labels of source domain data, we choose to adjust data distribution of the source domain. To achieve this purpose, we introduce a trainable class weight $\bm{w}$ to reweigh source domain examples by class when performing DIRL, with $\bm{w}_i > 0$. Specifically, we hope that:
\begin{equation*}
    \bm{w}_i\rm{P}_S(\rm{Y}=i) = \rm{P}_T(\rm{Y}=i),
\end{equation*}
and we denote $\bm{w}^*$ the value of $\bm{w}$ that makes this equation hold. We shall see that when $\bm{w}=\bm{w}^*$, DIRL is to align
$\rm{P}_S(G(\rm{X})|\rm{Y})$ with $\rm{P}_T(G(\rm{X})|\rm{Y})$ without the shift of $\rm{P}(\rm{Y})$.
According to our analysis, we know that due to the shift of $\rm{P}(\rm{Y})$, there is a conflict between the training objects of the supervised learning $\mathcal{L}_{sup}$ and the domain-invariant learning $\mathcal{L}_{inv}$. And the conflict degree will decrease as $\rm{P}_S(\rm{Y})$ getting close to $\rm{P}_T(\rm{Y})$. Therefore, during model training, $\bm{w}$ is expected to be optimized toward $\bm{w}^*$ since it will make $\rm{P}(\rm{Y})$ of the weighted source domain close to $\rm{P}_T(\rm{Y})$, so as to solve the conflict. 


We now show how to transfer existing DIRL models to their WDIRL counterparts with the above idea. Let $\mathbb{S}:\rm{P} \rightarrow {R}$ denote a statistic function defined over a distribution $\rm{P}$. For example, the expectation function $\mathbb{E}(\rm{X})$ in $\mathbb{E}(\rm{X}_S) \equiv \mathbb{E}(\rm{X})(\rm{P}_S(\rm{X}))$ is a concrete instaintiation of $\mathbb{S}$. In general, to transfer models from DIRL to WDIRL, we should replace $\mathbb{S}(\rm{P}_S(\rm{X}))$ defined in $\mathcal{L}_{inv}$ with \begin{align*}
    \hat{\rm{P}}_S(\rm{X})=\sum_{i=1}^L \bm{w}_i \rm{P}_S(\rm{Y}=i) \mathbb{S}(\rm{P}_S(\rm{X}|\rm{Y}=i)), \\ 
    s.t., \bm{w}_i > 0, \sum_{i=1}^L \bm{w}_i \rm{P}_S(\rm{Y}=i) = 1.
\end{align*}

Take the CMD metric as an example. In WDIRL, the revised form of ${\text{CMD}}_K$ is defined by:
\begin{equation}
\begin{split}
&\widehat{\text{CMD}}_K(\rm{X}_S, \rm{X}_T) \\
& \resizebox{\columnwidth}{!}{$= \frac{1}{|b-a|} \parallel \sum_{i=1}^L \bm{w}_i \rm{P}_S(\rm{Y}=i) \mathbb{E}(\rm{X}_S|\rm{Y}_S=i) - E(\rm{X}_T) \parallel_2$} \\
&\resizebox{\columnwidth}{!}{$+ \frac{1}{|b-a|^k} \sum_{k=2}^{K} \parallel \sum_{i=1}^L \bm{w}_i \rm{P}_S(\rm{Y}=i) C_k(\rm{X}_S|\rm{Y}_S=i) - C_k(\rm{X}_T) \parallel_2$}, \\
& s.t., \bm{w}_i > 0, \sum_{i=1}^L \bm{w}_i \rm{P}_S(\rm{Y}=i) = 1.
\end{split}
\end{equation}
Here, $\mathbb{E}(\rm{X}_S|\rm{Y}_S=i) \equiv \mathbb{E}(\rm{X})(\rm{P}_S(\rm{X}|\rm{Y}=i))$ denotes the expectation of $\rm{X}$ over distribution $\rm{P}_S(\rm{X}|\rm{Y}=i)$. Note that both $\rm{P}_S(\rm{Y}=i)$ and $\mathbb{E}(\rm{X}_S|\rm{Y}_S=i)$ can be estimated using source labeled data, and $\mathbb{E}(\rm{X}_T)$ can be estimated using target unlabeled data. 

As for those adversarial-learning-based DIRL methods, e.g., DANN \cite{ganin2016domain}, the revised domain-invariant loss can be precisely defined by:
\begin{equation}
\begin{split}
\hat{\mathcal{L}}_d &\resizebox{0.9\columnwidth}{!}{$= \sum_{i=1}^L \bm{w}_i \rm{P}_S(\rm{Y}=i) \mathbb{E}_{x \sim \rm{P}_S(\rm{X}|\rm{Y}=i)}\left[\log \frac{1}{D(G(x))}\right]$} \\
&+ \mathbb{E}_{x \sim \rm{P}_T(\rm{X})}\left[\log\frac{1}{1- D(G(x))}\right],\\
& s.t., \bm{w}_i > 0, \sum_{i=1}^L \bm{w}_i \rm{P}_S(\rm{Y}=i) = 1.
\end{split}
\end{equation}
During model training, $D$ is optimized in the direction to minimize $\hat{\mathcal{L}}_d$, while $G$ and $\bm{w}$ are optimized to maximize $\hat{\mathcal{L}}_d$. 
In the following, we denote $\widehat{\text{JSD}}(\rm{P}_S, \rm{P}_T)$ the equivalent loss defined over $G$ for the revised version of domain adversarial learning. 

The general task loss in WDIRL is defined by:
\begin{equation} \label{eq:weighted_general_loss}
\hat{\mathcal{L}} = \mathcal{L}_{sup}(\mathcal{D}_S) + \alpha \hat{\mathcal{L}}_{inv}(\mathcal{D}_S, \mathcal{D}_T),
\end{equation}
where $\hat{\mathcal{L}}_{inv}$ is a unified representation of the domain-invariant loss in WDIRL, such as $\widehat{\text{CMD}}_K$ and $\widehat{\text{JSD}}(\rm{P}_S, \rm{P}_T)$.

\subsection{Align $\rm{P}(\rm{Y}|\rm{X})$ with Class Weight} \label{sec:p_y}

In the above step, we align $\rm{P}(\rm{X}|\rm{Y})$ across domains by performing domain-invariant learning on the class-weighted source domain and the original target domain. In this step, we deal with the shift of $\rm{P}(\rm{Y})$. Suppose that we have successfully resolved the shift of $\rm{P}(\rm{X}|\rm{Y})$ with $G$, i.e., $\rm{P}_S(G(\rm{X})|\rm{Y})=\rm{P}_T(G(\rm{X})|\rm{Y})$. Then, according to the work of \cite{chan2005word}, we have:
\begin{equation} \label{eq:true_p_t_y}
\resizebox{1.\columnwidth}{!}{$\rm{P}_T(\rm{Y}=i|G(\rm{X}))=\frac{\gamma(\rm{Y}=i)\rm{P}_S(\rm{Y}=i|G(\rm{X}))}{\sum_{j =1}^{L} \gamma(\rm{Y}=j)\rm{P}_S(\rm{Y}=j|G(\rm{X}))},$}
\end{equation}
where $\gamma(\rm{Y}=i)={\rm{P}_T(\rm{Y}=i)}/{\rm{P}_S(\rm{Y}=i)}$. Of course, in most of the real-world tasks, we do not know the value of $\gamma(\rm{Y}=i)$. However, note that $\gamma(\rm{Y}=i)$ is exactly the expected class weight $\bm{w}^*_i$. Therefore, a natural practice of this step is to estimate $\gamma(\rm{Y}=i)$ with the obtained $\bm{w}_i$ in the first step and estimate $\rm{P}_T(\rm{Y}|G(\rm{X}))$ with:
\begin{equation} \label{eq:estimate_p_t_y}
\resizebox{1\columnwidth}{!}{$\rm{P}_T(\rm{Y}=i|G(\rm{X})) \leftarrow \frac{\bm{w}_i\rm{P}_S(\rm{Y}=i|G(\rm{X}))}{\sum_{j =1}^{L} \bm{w}_j \rm{P}_S(\rm{Y}=j|G(\rm{X}))}.$}
\end{equation}

In summary, to transfer methods of the DIRL paradigm to WDIRL, we should: first revise the definition of $\mathcal{L}_{inv}$, obtaining its corresponding WDIRL form $\hat{\mathcal{L}}_{inv}$; then perform supervised learning and domain-invariant representation learning on $\mathcal{D}_S$ and $\mathcal{D}_T$ according to Eq. (\ref{eq:weighted_general_loss}), obtaining a supervised classifier $\rm{P}_S(\rm{Y}|\rm{X}; \bm{\Phi})$ and a class weight vector $\bm{w}$; and finally, adjust $\rm{P}_S(\rm{Y}|\rm{X}; \bm{\Phi})$ using $\bm{w}$ according to Eq. (\ref{eq:estimate_p_t_y}) and obtain the target domain classifier $\rm{P}_T(\rm{Y}|\rm{X}; \bm{\Phi})$.


\section{Experiment}

\subsection{Experiment Design}
Through the experiments, we empirically studied our analysis on DIRL and the effectiveness of our proposed solution in dealing with the problem it suffered from. In addition, we studied the impact of each step described in \textsection \ref{sec:p_x_y} and \textsection \ref{sec:p_y} to our proposed solution, respectively. To performe the study, we carried out performance comparison between the following models:
\begin{itemize}
\item \textbf{SO}: the source-only model trained using source domain labeled data without any domain adaptation.

\item \textbf{CMD}: the centre-momentum-based domain adaptation model \cite{zellinger2017central} of the original DIRL framework that implements $\mathcal{L}_{inv}$ with $\text{CMD}_K$.

\item \textbf{DANN}: the adversarial-learning-based domain adaptation model \cite{ganin2016domain} of the original DIRL framework that implements $\mathcal{L}_{inv}$ with $\text{JSD}(\rm{P}_S, \rm{P}_T)$.

\item \textbf{$\text{CMD}^\dagger$}: the weighted version of the CMD model that only applies the first step (described in \textsection \ref{sec:p_x_y}) of our proposed method.

\item \textbf{$\text{DANN}^\dagger$}: the weighted version of the DANN model that only applies the first step of our proposed method.

\item \textbf{$\text{CMD}^{\dagger \dagger}$}: the weighted version of the CMD model that applies both the first and second (described in \textsection \ref{sec:p_y}) steps of our proposed method.

\item \textbf{$\text{DANN}^{\dagger \dagger}$}: the weighted version of the DANN model that applies both the first and second steps of our proposed method.

\item \textbf{$\text{CMD}^{*}$}: a variant of {$\text{CMD}^{\dagger \dagger}$} that assigns $\bm{w}^*$ (estimate from target labeled data) to $\bm{w}$ and fixes this value during model training. 

\item \textbf{$\text{DANN}^{*}$}: a variant of {$\text{DANN}^{\dagger \dagger}$} that assigns $\bm{w}^*$ to $\bm{w}$ and fixes this value during model training. 
\end{itemize}

Intrinsically, {SO} can provide an empirical lowerbound for those domain adaptation methods. $\text{CMD}^{*}$ and $\text{DANN}^{*}$ can provide the empirical upbound of $\text{CMD}^{\dagger\dagger}$ and $\text{DANN}^{\dagger\dagger}$, respectively. In addition, by comparing performance of $\text{CMD}^{*}$ and $\text{DANN}^{*}$ with that of $\text{SO}$, we can know the effectiveness of the DIRL framework when $\rm{P}(\rm{Y})$ dose not shift across domains. By comparing $\text{CMD}^\dagger$ with $\text{CMD}$, or comparing $\text{DANN}^\dagger$ with $\text{DANN}$, we can know the effectiveness of the first step of our proposed method. By comparing {$\text{CMD}^{\dagger\dagger}$} with $\text{CMD}^{\dagger}$, or comparing  {$\text{DANN}^{\dagger\dagger}$} with {$\text{DANN}^{\dagger}$}, we can know the impact of the second step of our proposed method. And finally, by comparing {$\text{CMD}^{\dagger\dagger}$} with $\text{CMD}$, or comparing {$\text{DANN}^{\dagger\dagger}$} with $\text{DANN}$, we can know the general effectiveness of our proposed solution.

\begin{table*}[t]
\renewcommand{\arraystretch}{1.15}
\setlength\arrayrulewidth{1pt} 
\setlength\tabcolsep{2pt} 
\centering
\resizebox{\textwidth}{!}{
\begin{tabular}{l|>{\columncolor[gray]{0.92}}c|c>{\columncolor[gray]{0.92}}cc>{\columncolor[gray]{0.92}}c|c>{\columncolor[gray]{0.92}}cc>{\columncolor[gray]{0.92}}c}
\hline
S$\rightarrow$T &SO  & $\text{CMD}$ & $\text{CMD}^\dagger$ & $\text{CMD}^{\dagger\dagger}$ & $\text{CMD}^{*}$ & DANN & $\text{DANN}^\dagger$ & $\text{DANN}^{\dagger\dagger}$ & $\text{DANN}^{*}$ \\ \hline
B$\rightarrow$D		& 83.52 $\pm$ 0.20 & 79.18 $\pm$ 0.28 & 82.01 $\pm$ 0.54 & 83.89 $\pm$ 0.65 & 84.83	$\pm$ 0.05 & 80.47 $\pm$ 0.52 & 84.53 $\pm$ 0.52 & 84.60 $\pm$ 0.18 & 84.33 $\pm$ 0.15\\ 
B$\rightarrow$E		& 81.83 $\pm$ 0.06 & 78.11 $\pm$ 0.19 & 84.02 $\pm$ 0.37 & 84.01 $\pm$ 0.45 & 84.26	$\pm$ 0.09 & 76.26 $\pm$ 1.16 & 84.75 $\pm$ 0.44 & 83.91 $\pm$ 0.58 & 83.71 $\pm$ 0.60\\ 
B$\rightarrow$K		& 82.72 $\pm$ 0.02 & 80.19 $\pm$ 0.12 & 83.91 $\pm$ 0.24 & 85.49 $\pm$ 0.05 & 85.49	$\pm$ 0.06 & 79.66 $\pm$ 0.49 & 82.64 $\pm$ 0.59 & 83.32 $\pm$ 0.27 & 84.87 $\pm$ 0.41\\
D$\rightarrow$B		& 82.97 $\pm$ 0.06 & 81.47 $\pm$ 0.38 & 83.20 $\pm$ 0.10 & 83.10 $\pm$ 0.12 & 83.11	$\pm$ 0.03 & 82.08 $\pm$ 0.97 & 83.10 $\pm$ 0.38 & 82.65 $\pm$ 0.08 & 82.05 $\pm$ 0.22\\ 
D$\rightarrow$E		& 81.97 $\pm$ 0.07 & 80.35 $\pm$ 0.03 & 82.48 $\pm$ 0.29 & 83.47 $\pm$ 0.12 & 83.57	$\pm$ 0.03 & 78.75 $\pm$ 0.54 & 83.01 $\pm$ 0.44 & 83.29 $\pm$ 0.51 & 83.09 $\pm$ 0.48\\
D$\rightarrow$K		& 83.51 $\pm$ 0.10 & 82.99 $\pm$ 0.22 & 86.94 $\pm$ 0.18 & 86.40 $\pm$ 0.23 & 86.34	$\pm$ 0.15 & 81.54 $\pm$ 0.70 & 85.05 $\pm$ 0.51 & 85.84 $\pm$ 0.71 & 86.06 $\pm$ 0.61\\ 
E$\rightarrow$B		& 80.65 $\pm$ 0.11 & 78.09 $\pm$ 0.34 & 79.65 $\pm$ 0.40 & 81.35 $\pm$ 0.31 & 81.82	$\pm$ 0.07 & 78.94 $\pm$ 0.73 & 80.70 $\pm$ 0.94 & 81.63 $\pm$ 0.74 & 81.53 $\pm$ 0.33\\
E$\rightarrow$D		& 80.25 $\pm$ 0.25 & 77.16 $\pm$ 1.99 & 80.07 $\pm$ 0.49 & 82.20 $\pm$ 0.17 & 81.85	$\pm$ 0.08 & 76.87 $\pm$ 0.50 & 79.73 $\pm$ 0.77 & 81.24 $\pm$ 0.47 & 82.04 $\pm$ 0.15\\ 
E$\rightarrow$K		& 87.43 $\pm$ 0.06 & 83.76 $\pm$ 0.15 & 86.87 $\pm$ 0.28 & 88.68 $\pm$ 0.13 & 89.00	$\pm$ 0.02 & 84.37 $\pm$ 0.89 & 87.89 $\pm$ 0.28 & 88.31 $\pm$ 0.36 & 88.38 $\pm$ 0.31\\
K$\rightarrow$B		& 80.05 $\pm$ 0.26 & 75.44 $\pm$ 0.37 & 81.00 $\pm$ 0.25 & 82.35 $\pm$ 0.16 & 82.34	$\pm$ 0.13 & 75.81 $\pm$ 0.21 & 80.97 $\pm$ 0.72 & 81.83 $\pm$ 0.32 & 81.13 $\pm$ 0.52\\ 
K$\rightarrow$D		& 79.88 $\pm$ 0.13 & 73.52 $\pm$ 0.27 & 79.85 $\pm$ 0.15 & 83.58 $\pm$ 0.05 & 83.64	$\pm$ 0.06 & 74.27 $\pm$ 0.82 & 80.49 $\pm$ 0.07 & 83.11 $\pm$ 0.76 & 83.53 $\pm$ 0.10\\
K$\rightarrow$E		& 87.30 $\pm$ 0.02 & 81.73 $\pm$ 0.46 & 87.80 $\pm$ 0.13 & 87.87 $\pm$ 0.04 & 88.04	$\pm$ 0.01 & 82.19 $\pm$ 0.00 & 87.52 $\pm$ 0.26 & 87.55 $\pm$ 0.18 & 87.80 $\pm$ 0.18\\
\hline \hline
Ave & 82.67 $\pm$ 0.11 & 79.33 $\pm$ 0.40 & 83.15 $\pm$ 0.37 & 84.36 $\pm$ 0.21 & 84.52 $\pm$ 0.07 & 79.42 $\pm$ 0.63 & 83.28 $\pm$ 0.49 & 83.32 $\pm$ 0.43 & 84.04 $\pm$ 0.34 \\ \hline
\end{tabular}
}
\caption{Mean accuracy $\pm$ standard deviation over five runs on the 12 binary-class cross-domain tasks.}
\label{table:binary_result}
\end{table*}

\subsection{Dataset and Task Design}
We conducted experiments on the Amazon reviews dataset
\cite{blitzer2007biographies}, which is a benchmark dataset in the cross-domain sentiment analysis field. This dataset contains Amazon product reviews of four different product domains: Books (B), DVD (D), Electronics (E), and Kitchen (K) appliances. Each review is originally associated with a rating of 1-5 stars and is encoded in 5,000 dimensional feature vectors of bag-of-words unigrams and bigrams. 

\paragraph{Binary-Class.} From this dataset, we constructed 12 binary-class cross-domain sentiment analysis tasks: B$\rightarrow$D, B$\rightarrow$E, B$\rightarrow$K, D$\rightarrow$B, D$\rightarrow$E, D$\rightarrow$K, E$\rightarrow$B, E$\rightarrow$D, E$\rightarrow$K, K$\rightarrow$B, K$\rightarrow$D, K$\rightarrow$E. Following the setting of previous works, we treated a reviews as class `1' if it was ranked up to 3 stars, and as class `2' if it was ranked 4 or 5 stars.
For each task, $\mathcal{D}_S$ consisted of 1,000 examples of each class, and $\mathcal{D}_T$ consists of 1500 examples of class `1' and 500 examples of class `2'. In addition, since it is reasonable to assume that $\mathcal{D}_T$ can reveal the distribution of target domain data, we controlled the target domain testing dataset to have the same class ratio as $\mathcal{D}_T$. Using the same label assigning mechanism, we also studied model performance over different degrees of $\rm{P}(\rm{Y})$ shift, which was evaluated by the max value of $\rm{P}_S(\rm{Y}=i)/\rm{P}_T(\rm{Y}=i), \forall i=1, \cdots, L$. Please refer to Appendix \textbf{C} for more detail about the task design for this study. 

\paragraph{Multi-Class.} We additionally constructed 12 multi-class cross-domain sentiment classification tasks. Tasks were designed to distinguish reviews of 1 or 2 stars (class 1) from those of 4 stars (class 2) and those of 5 stars (class 3). For each task, $\mathcal{D}_S$ contained 1000 examples of each class, and $\mathcal{D}_T$ consisted of 500 examples of class 1, 1500 examples of class 2, and 1000 examples of class 3. Similarly, we also controlled the target domain testing dataset to have the same class ratio as $\mathcal{D}_T$.

\subsection{Implementation Detail}

For all studied models, we implemented $G$ and $f$ using the same architectures as those in \cite{zellinger2017central}. For those DANN-based methods (i.e., DANN, $\text{DANN}^{\dagger}$, $\text{DANN}^{\dagger\dagger}$, and $\text{DANN}^{*}$), we implemented the discriminator $D$ using a 50 dimensional hidden layer with relu activation functions and a linear classification layer. Hyper-parameter $K$ of $\text{CMD}_K$ and $\widehat{\text{CMD}}_K$ was set to 5 as suggested by \citet{zellinger2017central}. Model optimization was performed using RmsProp \cite{tieleman2012lecture}. Initial learning rate of $\bm{w}$ was set to 0.01, while that of other parameters was set to 0.005 for all tasks. 

Hyper-parameter $\alpha$ was set to 1 for all of the tested models. We searched for this value in range $\alpha=[1, \cdots, 10]$ on task B $\rightarrow$ K. Within the search, label distribution was set to be uniform, i.e., $\rm{P}(\rm{Y}=i)=1/L$, for both domain B and K. We chose the value that maximize the performance of CMD on testing data of domain K. You may notice that this practice conflicts with the setting of unsupervised domain adaptation that we do not have labeled data of the target domain for training or developing. However, we argue that this practice would not make it unfair for model comparison since all of the tested models shared the same value of $\alpha$ and $\alpha$ was not directly fine-tuned on any tested task. With the same consideration, for every tested model, we reported its best performance achieved on testing data of the target domain during its training\footnote{Please refer to the attached source code in the appendix for more implementation detail of this work.}.
\begin{figure}[t]
\centering
\includegraphics[width=1\columnwidth]{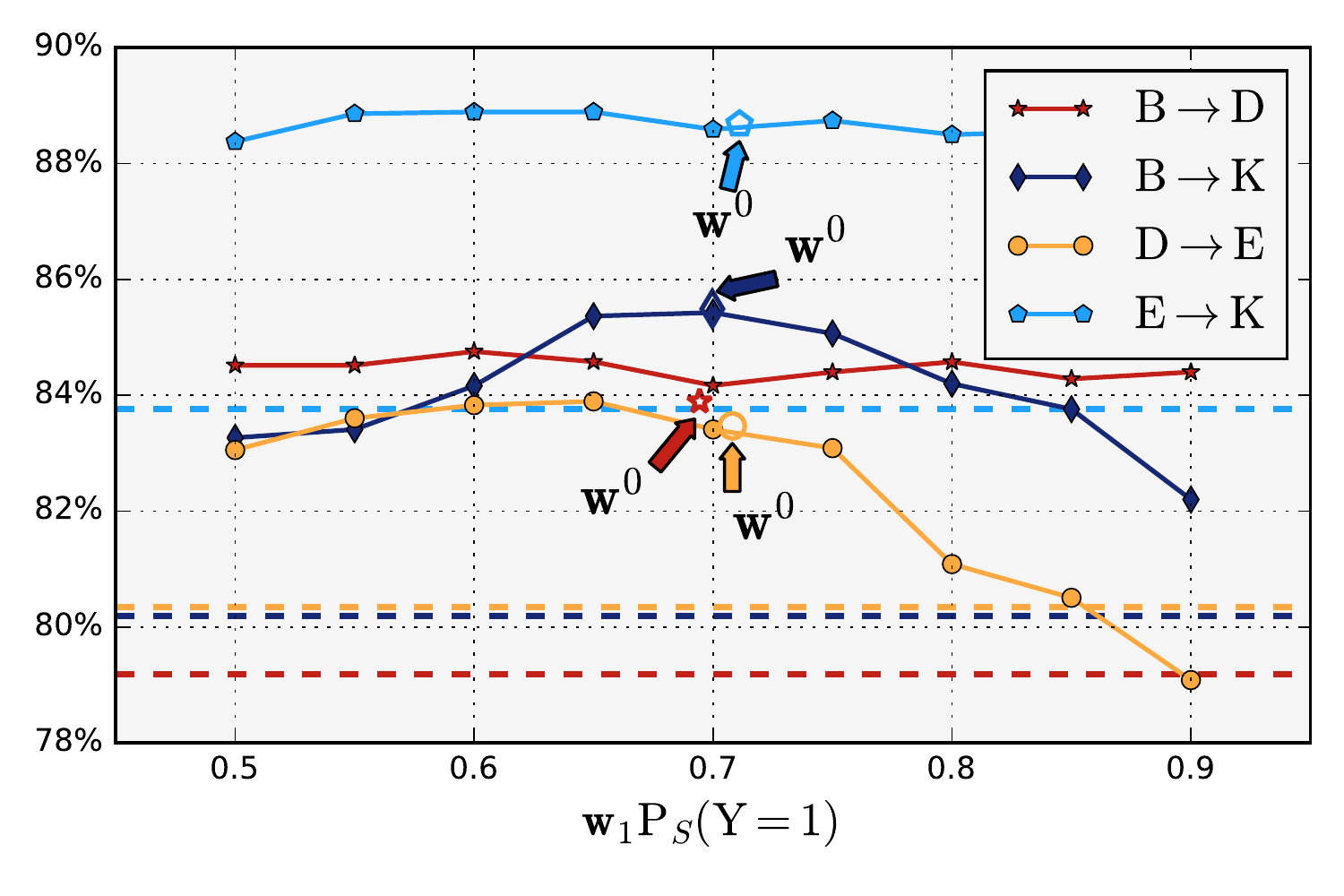}
\caption{Mean accuracy of WCMD$^{\dagger\dagger}$ over different initialization of $\bm{w}$. The empirical optimum value of $\bm{w}$ makes $\bm{w}_1 \rm{P}_S(\rm{Y}=1) = 0.75$. The dot line in the same color denotes performance of the CMD model and `$\bm{w}^0$' annotates performance of WCMD$^{\dagger\dagger}$ when initializing $\bm{w}$ with $\bm{w}^0$.}
\label{fig:wcmd_by_w}
\vspace*{-\baselineskip}
\end{figure}

To initialize $\bm{w}$, we used label prediction of the source-only model. Specifically, let $\rm{P}_{SO}(\rm{Y}|\rm{X}; \bm{\theta}_{SO})$ denote the trained source-only model. We initialized $\bm{w}_i$ by:
\begin{equation*}
\bm{w}^0_i = \frac{\frac{1}{|\mathcal{D}_T|}\sum_{x \in \mathcal{D}_T} \rm{P}_{SO}(y=i|x; \bm{\theta}_{SO})}{\frac{1}{|\mathcal{D}_S|} \sum_{y \in \mathcal{D}_S} \mathbb{I}(y=i)}.
\end{equation*}
Here, $\mathbb{I}$ denotes the indication function. To offer an intuitive understanding to this strategy, we report performance of WCMD$^{\dagger\dagger}$ over different initializations of $\bm{w}$ on 2 within-group (B$\rightarrow$D, E$\rightarrow$K) and 2 cross-group (B$\rightarrow$K, D$\rightarrow$E) binary-class domain adaptation tasks in Figure \ref{fig:wcmd_by_w}. Here, we say that domain B and D are of a group, and domain E and K are of another group since B and D are similar, as are E and K, but the two groups are different from one another \cite{blitzer2007biographies}. Note that $\rm{P}_{S}(\rm{Y}=1)=0.5$ is a constant, which is estimated using source labeled data. 
From the figure, we can obtain three main observations. First, WCMD$^{\dagger\dagger}$ generally outperformed its CMD counterparts with different initialization of $\bm{w}$. Second, it was better to initialize $\bm{w}$ with a relatively balanced value, i.e., $\bm{w}_i \rm{P}_S(\rm{Y}=i) \rightarrow \frac{1}{L}$ (in this experiment, $L=2$).
Finally, $\bm{w}^0$ was often a good initialization of $\bm{w}$, indicating the effectiveness of the above strategy.

\begin{table}[]
\setlength\arrayrulewidth{1pt} 
\setlength\tabcolsep{2pt} 
\centering
\resizebox{1.\columnwidth}{!}{
\begin{tabular}{l|>{\columncolor[gray]{0.92}}cc>{\columncolor[gray]{0.92}}cc}
\hline
Model  							& B$\rightarrow$D 	& B$\rightarrow$K 	& D$\rightarrow$E 	& E$\rightarrow$K \\ \hline \hline 
SO  							& 59.10 $\pm$ 0.83	& 60.77 $\pm$ 1.47	& 57.50 $\pm$ 0.67	& 66.13 $\pm$ 4.09\\ \hline
CMD 							& 59.11 $\pm$ 0.70	& 60.35 $\pm$ 1.32	& 56.59 $\pm$ 1.00	& 62.78 $\pm$ 3.16\\
$\text{CMD}^{\dagger}$ 			& 59.16 $\pm$ 1.00	& 61.32 $\pm$ 1.67	& 58.32 $\pm$ 1.89	& 64.94 $\pm$ 3.91\\
$\text{CMD}^{\dagger\dagger}$ 	& 60.69 $\pm$ 0.82	& 61.18 $\pm$ 1.84	& 60.12 $\pm$ 0.89	& 66.65 $\pm$ 3.77\\
$\text{CMD}^{*}$				& 60.26 $\pm$ 0.76 	& 61.77 $\pm$ 1.43 	& 59.84 $\pm$ 0.84 	& 66.42 $\pm$ 3.70\\
\hline
DANN 							& 59.16 $\pm$ 0.60	& 61.85 $\pm$ 0.64	& 57.80 $\pm$ 0.32	& 65.50 $\pm$ 0.53\\
$\text{DANN}^{\dagger}$ 		& 60.07 $\pm$ 0.39	& 62.71 $\pm$ 0.34	& 59.97 $\pm$ 0.49	& 66.86 $\pm$ 3.23\\
$\text{DANN}^{\dagger\dagger}$ 	& 59.32 $\pm$ 0.52	& 63.07 $\pm$ 0.51	& 58.95 $\pm$ 0.32	& 66.54 $\pm$ 3.24\\
$\text{DANN}^{*}$ 				& 60.49 $\pm$ 0.17 	& 62.90 $\pm$ 0.39 	& 58.89 $\pm$ 0.37 	& 66.45 $\pm$ 3.23\\
\hline
\end{tabular}
}
\caption{Mean accuracy $\pm$ standard deviation over five runs on the 2 within-group and 2 cross-group multi-class domain-adaptation tasks.}
\label{table:multi_result}
\end{table}

\begin{figure*}[t]
\includegraphics[width=\textwidth]{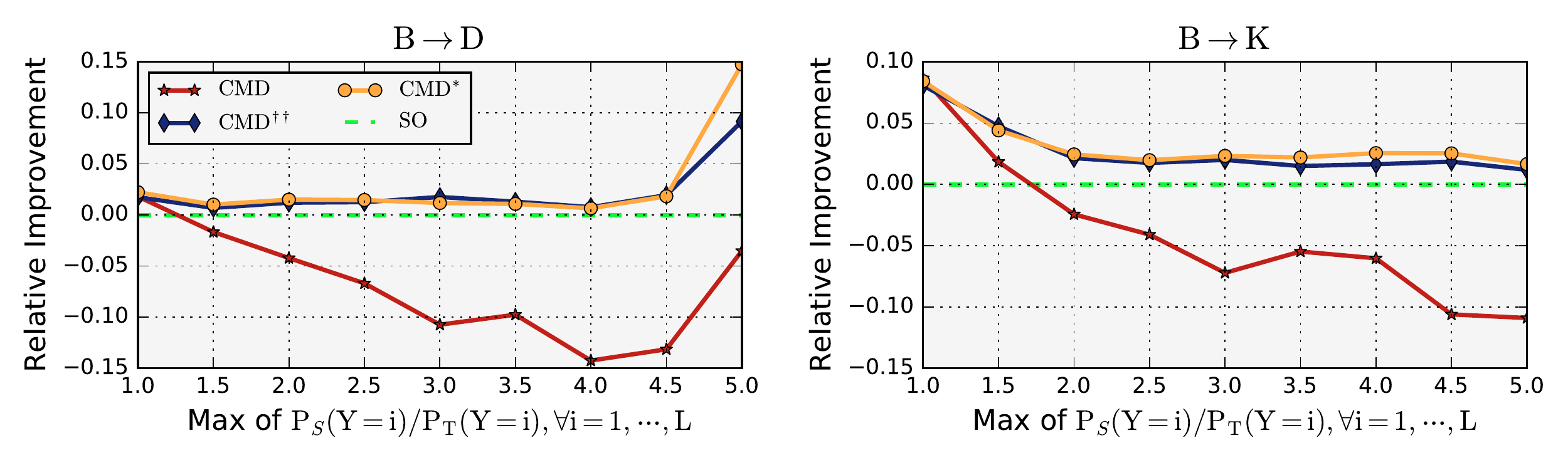}
\caption{Relative improvement over the SO baseline under different degrees of $\rm{P}(\rm{Y})$ shift on the B$\rightarrow$D and B $\rightarrow$K binary-class domain adaptation tasks.}
\label{fig:target_shift_degree}
\end{figure*}

\subsection{Main Result}

Table \ref{table:binary_result} shows model performance on the 12 binary-class cross-domain tasks. From this table, we can obtain the following observations. 
\textbf{First}, CMD and DANN underperform the source-only model (SO) on all of the 12 tested tasks, indicating that DIRL in the studied situation will degrade the domain adaptation performance rather than improve it. This observation confirms our analysis. 
\textbf{Second}, {$\text{CMD}^{\dagger\dagger}$} consistently outperformed CMD and SO. This observation shows the effectiveness of our proposed method for addressing the problem of the DIRL framework in the studied situation. Similar conclusion can also be obtained by comparing performance of {$\text{DANN}^{\dagger\dagger}$} with that of DANN and SO.
\textbf{Third}, $\text{CMD}^{\dagger}$ and $\text{DANN}^{\dagger}$ consistently outperformed $\text{CMD}$ and DANN, respectively, which shows the effectiveness of the first step of our proposed method. 
\textbf{Finally}, on most of the tested tasks, $\text{CMD}^{\dagger\dagger}$ and $\text{DANN}^{\dagger\dagger}$ outperforms $\text{CMD}^{\dagger}$ and $\text{DANN}^{\dagger}$, respectively. 

Figure \ref{fig:target_shift_degree} depicts the relative improvement, e.g., $(\text{Acc}(\text{CMD})-\text{Acc}(\text{SO}))/\text{Acc}(\text{SO})$, of the domain adaptation methods over the SO baseline under different degrees of $\rm{P}(\rm{Y})$ shift, on two binary-class domain adaptation tasks (You can refer to Appendix \textbf{C} for results of the other models on other tasks). From the figure, we can see that the performance of CMD generally got worse as the increase of $\rm{P}(\rm{Y})$ shift. In contrast, our proposed model $\text{CMD}^{\dagger\dagger}$ performed robustly to the varying of $\rm{P}(\rm{Y})$ shift degree. Moreover, it can achieve the near upbound performance characterized by $\text{CMD}^{*}$. This again verified the effectiveness of our solution. 

Table \ref{table:multi_result} reports model performance on the 2 within-group (B$\rightarrow$D, E$\rightarrow$K) and the 2 cross-group (B$\rightarrow$K, D$\rightarrow$E) multi-class domain adaptation tasks (You can refer to Appendix \textbf{D} for results on the other tasks). From this table, we observe that on some tested tasks, $\text{CMD}^{\dagger\dagger}$ and $\text{DANN}^{\dagger\dagger}$ did not greatly outperform or even slightly underperformed $\text{CMD}^{\dagger}$ and $\text{DANN}^{\dagger}$, respectively. A possible explanation of this phenomenon is that the distribution of $\mathcal{D}_T$ also differs from that of the target domain testing dataset. Therefore, the estimated or learned value of $\bm{w}$ using $\mathcal{D}_T$ is not fully suitable for application to the testing dataset. 
This explanation is verified by the observation that $\text{CMD}^{\dagger}$ and $\text{DANN}^{\dagger}$ also slightly outperforms $\text{CMD}^{*}$ and $\text{DANN}^{*}$ on these tasks, respectively.

\section{Conclusion}

In this paper, we studied the problem of the popular domain-invariant representation learning (DIRL) framework for domain adaptation, when $\rm{P}(\rm{Y})$ changes across domains. To address the problem, we proposed a weighted version of DIRL (WDIRL). We showed that existing methods of the DIRL framework can be easily transferred to our WDIRL framework. Extensive experimental studies on benchmark cross-domain sentiment analysis datasets verified our analysis and showed the effectiveness of our proposed solution.

\bibliographystyle{acl_natbib}
\bibliography{reference}

\end{document}